\documentclass{article} 
\usepackage{nips10submit_e,times,algorithm,algorithmic}
\usepackage{amsmath,amssymb,amsthm,bm,multicol,longtable,float,dsfont,enumerate,epsfig,multirow,fullpage,url}

\title{Stochastic ADMM for Nonsmooth Optimization}

\author{
Hua Ouyang\\
Computational Science and Engineering\\
Georgia Institute of Technology\\
\texttt{houyang@cc.gatech.edu}
\And
Niao He\\
Industrial and Systems Engineering\\
Georgia Institute of Technology\\
\texttt{nhe6@isye.gatech.edu}
\And
Alexander Gray\\
Computational Science and Engineering\\
Georgia Institute of Technology\\
\texttt{agray@cc.gatech.edu}
}

\nipsfinalcopy 

\newtheorem{assumption}[]{Assumption}
\newtheorem{remark}[]{Remark}

\newtheorem{lemma}[]{Lemma}

\newtheorem{theorem}[]{Theorem}

\newtheorem{claim}[]{Claim}

\def \Pr{\textrm{Prob}}
\def \E{\mathbb{E}}

\begin{document}

\maketitle

\begin{abstract}
We present a stochastic setting for optimization problems with nonsmooth convex separable objective functions  over linear equality constraints. To solve such problems, we propose a stochastic Alternating Direction Method of Multipliers (ADMM) algorithm. Our algorithm applies  to a more general class of nonsmooth convex functions that does not necessarily have a closed-form solution by minimizing the augmented function directly. We also demonstrate the rates of convergence for our algorithm under various structural assumptions of the stochastic functions: $O(1/\sqrt{t})$ for convex functions and $O(\log t/t)$ for strongly convex functions. Compared to previous literature,  we establish the convergence rate of ADMM algorithm, for the first time,  in terms of both the objective value and the feasibility violation.
\end{abstract}

\footnotetext[1]{A short version appears in the 5th NIPS Workshop on Optimization for Machine Learning, Lake Tahoe, Nevada, USA, 2012.}

\section{Introduction}
The Alternating Direction Method of Multipliers (ADMM) \cite{glowinski75admm,gabay76admm} is a very simple computational method for constrained optimization proposed in 1970s. The theoretical aspects of ADMM have been studied from 1980s to 90s and its global convergence was established in the literature \cite{gabay83ammvi,glowinski89alosmnm,eckestein92odrsmppa}. As reviewed in the comprehensive paper \cite{boyd10dosladmm}, with its capacity of dealing with objective functions separately and synchronously, this method turned out to be a natural fit in the field of large-scale data-distributed machine learning and big-data related optimization and therefore received significant amount of attention in the last few years. Intensive theoretical and practical advances are conducted thereafter. On the theoretical hand,  ADMM is recently shown to have a rate of convergence of $O(1/N)$ \cite{monteiro10icbba,he12ocrdradm, he12oncrdradmm,wang12oadm}, where $N$ stands for the number of iterations. On the practical hand, ADMM has been applied to a wide range of application domains, such as compressed sensing \cite{yang11adal1}, image restoration \cite{goldstein09sbml1}, video processing and matrix completion \cite{goldfarb10falmmstcf}. Besides that, many variations of this classical method have been recently developed, such as linearized \cite{goldfarb10falmmstcf,zhang11ladmm,yang12laladmnnm},  accelerated \cite{goldfarb10falmmstcf}, and online \cite{wang12oadm} ADMM. However, most of these variants including the classic one implicitly assume full accessibilty to true data values, while in reality one can hardly ignore the existence of noise. A more natural way of handling this issue is to consider unbiased or even biased observations of true data, which leads us to the stochastic setting.

\subsection{Stochastic Setting for ADMM}
In this work, we study a family of convex optimization problems in which our objective functions are separable and stochastic. In particular, we are interested in solving the following linear equality-constrained stochastic optimization:
\begin{equation}\label{eq:problem}
\min_{\mathbf{x}\in\mathcal{X},\mathbf{y}\in\mathcal{Y}} \mathbb{E}_{\boldsymbol{\xi}}\theta_1(\mathbf{x},\boldsymbol{\xi}) + \theta_2(\mathbf{y})\ \ \text{s.t. } A\mathbf{x}+ B\mathbf{y} = \mathbf{b},
\end{equation}
where $\mathbf{x}\in\mathbb{R}^{d_1}, \mathbf{y}\in\mathbb{R}^{d_2}, A\in\mathbb{R}^{m\times d_1}, B\in\mathbb{R}^{m\times d_2}, \mathbf{b}\in\mathbb{R}^{ m}$, $\mathcal{X}$ is a convex compact set, and $\mathcal{Y}$ is a closed convex set. We are able to draw a sequence of identical and independent (i.i.d.) observations from the random vector $\boldsymbol{\xi}$ that obeys a fixed but unknown distribution $P$.  One can see that when $\boldsymbol{\xi}$ is deterministic, we can recover the traditional problem setting for ADMM \cite{boyd10dosladmm}.  Denote the expectation function $\theta_1(\mathbf{x})\equiv \E_{\boldsymbol{\xi}}\theta_1(\mathbf{x},\boldsymbol{\xi})$. In our most general setting, real-valued functions $\theta_1(\cdot)$ and $\theta_2(\cdot)$ are convex but not necessarily continuously differentiable.

Note that our stochastic setting of the problem is quite different from that of the Online ADMM proposed in \cite{wang12oadm}. In Online ADMM, one does not assume $\boldsymbol{\xi}$ to be i.i.d., nor the objective to be stochastic, but instead, a deterministic concept referred as regret is concerned: $R\left(\mathbf{x}_{[1:t]}\right)\equiv \sum_{k=1}^t\left[ \theta_1(\mathbf{x}_k,\boldsymbol{\xi}_k) + \theta_2(\mathbf{y}_k)\right] - \inf_{A\mathbf{x}+B\mathbf{y}=\mathbf{b}}\sum_{k=1}^t \left[ \theta_1(\mathbf{x}, \boldsymbol{\xi}_k) + \theta_2(\mathbf{y}) \right]$.

\subsection{Our Contributions}
In this work, we propose a stochastic setting of the ADMM problem and design the Stochastic ADMM algorithm.  A key algorithmic feature of our Stochastic ADMM that distinguishes it from previous ADMM and variants is the first-order approximation of $\theta_1$ that we use to modify the augmented Lagrangian. This simple modification not only guarantees the convergence of our stochastic method, but also benefits to a more general class of convex objective functions which might not have a closed-form solution in minimizing the augmented $\theta_1$ directly. For example, with stochastic ADMM,  we can derive close-form updates for the nonsmooth hinge loss function (used in support vector machines). However, with deterministic ADMM, one has to call SVM solvers during each iteration \cite{boyd10dosladmm}, which is indeed very time-consuming. One of our main contributions is that we develop the convergence rates of our algorithm under various structural assumptions. For convex $\theta_1(\cdot)$, the rate is proved to be $O(1/\sqrt{t})$; for strongly convex $\theta_1(\cdot)$, the rate is proved to be $O(\log t/t)$. To the best of our knowledge, this is the first time that convergence rates of ADMM are established for both the objective value and the feasibility violation. By contrast, recent research \cite{he12ocrdradm,wang12oadm} only shows the convergence of ADMM indirectly in terms of the satisfaction of variational inequalities.

\subsection{Notations}
Throughout this paper, we denote the subgradients of $\theta_1$ and $\theta_2$ as $\theta_1^{\prime}$ and $\theta_2^{\prime}$. When they are differentiable, we will use $\nabla \theta_1$ and $\nabla \theta_2$ to denote the gradients. We use the notation $\theta_1$ both for the instance function value $\theta_1(\mathbf{x},\boldsymbol{\xi})$ and for its expectation $\theta_1(\mathbf{x})$. We denote by $\theta(\mathbf{u}) \equiv \theta_1(\mathbf{x}) + \theta_2(\mathbf{y})$  the sum of the stochastic and the deterministic functions. For simplicity and clarity, we will use the following notations to denote stacked vectors or tuples:
\begin{equation}\label{eq:notations}
\begin{split}
&\mathbf{u} \equiv \left( \begin{array}{c}
\mathbf{x} \\
\mathbf{y} \end{array} \right),\ \ 
\mathbf{w} \equiv \left( \begin{array}{c}
\mathbf{x} \\
\mathbf{y}\\
\boldsymbol{\lambda} \end{array} \right),\ \ 
\mathbf{w}_k\equiv \left( \begin{array}{c}
\mathbf{x}_k \\
\mathbf{y}_k\\
\boldsymbol{\lambda}_k\end{array} \right),\ \ 
\mathcal{W}\equiv \left( \begin{array}{c}
\mathcal{X} \\
\mathcal{Y}\\
\mathbb{R}^m \end{array} \right),\\
&\bar{\mathbf{u}}_k \equiv\left( \begin{array}{c}
\frac{1}{k}\sum_{i=0}^{k-1}\mathbf{x}_i \\
\frac{1}{k}\sum_{i=1}^{k}\mathbf{y}_i \end{array} \right), 
\ \bar{\mathbf{w}}_k \equiv \left( \begin{array}{c}
\frac{1}{k}\sum_{i=1}^{k}\mathbf{x}_i \\
\frac{1}{k}\sum_{i=1}^{k}\mathbf{y}_i\\
\frac{1}{k}\sum_{i=1}^{k}\boldsymbol{\lambda}_i\\
\end{array}\right),\ F(\mathbf{w}) \equiv \left( \begin{array}{c}
-A^T\boldsymbol{\lambda} \\
-B^T\boldsymbol{\lambda}\\
A\mathbf{x}+B\mathbf{y}-\mathbf{b}\end{array} \right).\\
\end{split}
\end{equation}
For a positive semidefinite matrix $G\in\mathbb{R}^{d_1\times d_1}$, we define the $G$-norm of a vector $\mathbf{x}$ as
$\|\mathbf{x}\|_{G}: = \|G^{1/2}\mathbf{x}\|_2 = \sqrt{\mathbf{x}^T G \mathbf{x}}$. We use $\left\langle \cdot,\cdot\right\rangle$ to denote the inner product in a finite dimensional Euclidean space. When there is no ambiguity, we often use $\|\cdot\|$ to denote the Euclidean norm $\|\cdot\|_2$. We assume that the optimal solution of (\ref{eq:problem}) exists and denote it as $\mathbf{u}_*\equiv \left( \mathbf{x}_*^T,
\mathbf{y}_*^T \right)^T $. The following quantity appear frequently in our convergence analysis:\begin{equation}\label{eq:additional_notations}
\begin{split}
&\boldsymbol{\delta}_{k} \equiv \theta_1^{\prime}(\mathbf{x}_{k-1},\boldsymbol{\xi}_{k}) -  \theta_1^{\prime}(\mathbf{x}_{k-1}), \\
&D_{\mathcal{X}}\equiv\sup_{\mathbf{x}_a,\mathbf{x}_b\in\mathcal{X}}\|\mathbf{x}_a-\mathbf{x}_b\|,\ \ D_{\mathbf{y}_*,B}\equiv\|B(\mathbf{y}_0-\mathbf{y}_*)\|.
\end{split}
\end{equation}

\subsection{Assumptions}
Before presenting the algorithm and convergence results, we list the assumptions that will be used in our statements.

\begin{assumption}\label{weak assumption}
For all $\mathbf{x}\in\mathcal{X}$, $\mathbb{E} \Big[ \| \theta_1^{\prime}(\mathbf{x},\boldsymbol{\xi})\|^2 \Big] \leq M^2$.
\end{assumption}

\begin{assumption}\label{strong assumption}
For all $\mathbf{x}\in\mathcal{X}$, $\mathbb{E}\Big[\exp\Big\{ \| \theta_1^{\prime}(\mathbf{x},\boldsymbol{\xi})\|^2 /M^2\Big\}\Big] \leq \exp\{1\}$.
\end{assumption}

\begin{assumption}\label{sigma_assumption}
For all $\mathbf{x}\in\mathcal{X}$, $\mathbb{E} \left[ \| \theta_1^{\prime}(\mathbf{x},\boldsymbol{\xi}) -  \theta_1^{\prime}(\mathbf{x})\|^2 \right] \leq \sigma^2$.
\end{assumption}

\section{Stochastic ADMM Algorithm}
Directly solving problem (\ref{eq:problem}) can be nontrivial  even if $\boldsymbol{\xi}$ is deterministic and the equality constraint is as simple as $\mathbf{x}-\mathbf{y}=\mathbf{0}$. For example, using the augmented Lagrangian method, one has to minimize the augmented Lagrangian:
\begin{equation*}
\min_{\mathbf{x}\in\mathcal{X},\mathbf{y}\in\mathcal{Y}}\mathcal{L}_{\beta}(\mathbf{x},\mathbf{y},\lambda) \equiv \min_{\mathbf{x}\in\mathcal{X},\mathbf{y}\in\mathcal{Y}} \theta_1(\mathbf{x}) + \theta_2(\mathbf{y}) - \left\langle\boldsymbol\lambda, A\mathbf{x}+B\mathbf{y}-\mathbf{b}\right\rangle + \frac{\beta}{2}\|A\mathbf{x}+B\mathbf{y}-\mathbf{b}\|^2,
\end{equation*}
where $\beta$ is a pre-defined penalty parameter. This problem is at least not easier than solving the original one. The (deterministic) ADMM (Alg.\ref{alg:deterministic_admm}) solves this problem in a Gauss-Seidel manner: minimizing $\mathcal{L}_\beta$ w.r.t. $\mathbf{x}$ and $\mathbf{y}$ alternatively given the other fixed, followed by a penalty update over the Lagrangian multiplier $\boldsymbol{\lambda}$.
\begin{algorithm}
\caption{\textsf{Deterministic ADMM}}
\label{alg:deterministic_admm}
\begin{algorithmic}
\STATE [0.] Initialize $\mathbf{y}_0$ and $\boldsymbol{\lambda}_0=0$.
\FOR{$k=0,1,2,\ldots$}
  \STATE[1.] $\mathbf{x}_{k+1}\leftarrow \arg\min_{\mathbf{x}\in\mathcal{X}} \left\{ \theta_1(\mathbf{x}) + \frac{\beta}{2}\left\|\left(A\mathbf{x}+B\mathbf{y}_k-\mathbf{b}\right) - \frac{\boldsymbol{\lambda}_k}{\beta}\right\|^2 \right\}$.
  \STATE[2.] $\mathbf{y}_{k+1}\leftarrow \arg\min_{\mathbf{y}\in\mathcal{Y}} \left\{ \theta_2(\mathbf{y}) + \frac{\beta}{2}\left\|\left(A\mathbf{x}_{k+1}+B\mathbf{y}-\mathbf{b}\right) - \frac{\boldsymbol{\lambda}_k}{\beta}\right\|^2 \right\}$.
  \STATE[3.] $\boldsymbol{\lambda}_{k+1} \leftarrow \boldsymbol{\lambda}_{k} - \beta\left(A\mathbf{x}_{k+1} + B\mathbf{y}_{k+1} -\mathbf{b} \right)$.
\ENDFOR
\end{algorithmic}
\end{algorithm}

A variant deterministic algorithm named linearized ADMM replaces Line 1 of Alg.\ref{alg:deterministic_admm} by
\begin{equation}\label{eq:linearized_ADMM}
\mathbf{x}_{k+1}\leftarrow \arg\min_{\mathbf{x}\in\mathcal{X}} \left\{ \theta_1(\mathbf{x}) + \frac{\beta}{2}\left\|\left(A\mathbf{x}+B\mathbf{y}_k-\mathbf{b}\right) - \boldsymbol{\lambda}_k/\beta \right\|^2 + \frac{1}{2}\|\mathbf{x}-\mathbf{x}_k\|_G^2 \right\},
\end{equation}
where $G\in\mathbb{R}^{d_1\times d_1}$ is positive semidefinite. This variant can be regarded as a generalization of the original ADMM. When $G=0$, it is the same as Alg.\ref{alg:deterministic_admm}. When $G=rI_{d_1}-\beta A^T A$, it is equivalent to the following linearized proximal point method:
\begin{equation*}
\mathbf{x}_{k+1} \leftarrow \arg\min_{\mathbf{x}\in\mathcal{X}}\left\{ \theta_1(\mathbf{x}) + 
\beta(\mathbf{x}-\mathbf{x}_k)^T\left[A^T(A\mathbf{x_k}+B\mathbf{y}_k-\mathbf{b}-\boldsymbol{\lambda}_k/\beta) \right] + \frac{r}{2}\|\mathbf{x}-\mathbf{x}_k\|^2 \right\}.
\end{equation*}
Note that the linearization is applied only to the quadratic function $\|(A\mathbf{x}+B\mathbf{y}_k-\mathbf{b})-\boldsymbol{\lambda}_k/\beta\|^2$, but not to $\theta_1$. This approximation helps when Line 1 of Alg.\ref{alg:deterministic_admm} does not produce a closed-form solution given the quadratic term. For example, let $\theta_1(\mathbf{x}) = \|\mathbf{x}\|_1$ and $A$ not identity.

As shown in Alg.\ref{alg:stochastic_admm}, we propose a \textit{Stochastic Alternating Direction Method of Multipliers} (\textit{Stochastic ADMM}) algorithm. Our algorithm shares some features with the classical and the linearized ADMM. One can see that Line 2 and 3 are essentially the same as before. However, there are two major differences in Line 1. First, we replace $\theta_1(\mathbf{x})$ with a first-order approximation of $\theta_1(\mathbf{x},\boldsymbol{\xi}_{k+1})$ at $\mathbf{x}_k$: $\theta_1(\mathbf{x}_k) + \mathbf{x}^T \theta_1^{\prime}(\mathbf{x}_k,\boldsymbol{\xi}_{k+1})$. This approximation has the same flavour of the stochastic mirror descent \cite{nemirovski09rsaasp} used for solving a one-variable stochastic convex problem. One important benefit of using this approximation is that our algorithm can be applied to nonsmooth objective functions, beyond the smooth and separable least squares loss used in lasso. Second, similar to the linearized ADMM (\ref{eq:linearized_ADMM}), we add an $l_2$-norm prox-function $\|\mathbf{x}-\mathbf{x}_k\|^2$ but scale it by a time-varying stepsize $\eta_{k+1}$. As we will see in Section \ref{sec:analysis}, the choice of  this stepsize is crucial in guaranteeing a convergence.

\begin{algorithm}
\caption{\textsf{Stochastic ADMM}}
\label{alg:stochastic_admm}
\begin{algorithmic}
\STATE [0.] Initialize $\mathbf{x}_0,\mathbf{y}_0$ and $\boldsymbol{\lambda}_0=0$.
\FOR{$k=0,1,2,\ldots$}
  \STATE[1.] $\mathbf{x}_{k+1}\leftarrow \arg\min_{\mathbf{x}\in\mathcal{X}} \left\{ \left\langle \theta_1^{\prime}(\mathbf{x}_k,\boldsymbol{\xi}_{k+1}), \mathbf{x} \right\rangle + \frac{\beta}{2}\left\|\left(A\mathbf{x}+B\mathbf{y}_k-\mathbf{b}\right) - \frac{\boldsymbol{\lambda}_k}{\beta}\right\|^2 + \frac{\|\mathbf{x}-\mathbf{x}_k\|^2}{2\eta_{k+1}} \right\}$.
  \STATE[2.] $\mathbf{y}_{k+1}\leftarrow \arg\min_{\mathbf{y}\in\mathcal{Y}} \left\{ \theta_2(\mathbf{y}) + \frac{\beta}{2}\left\|\left(A\mathbf{x}_{k+1}+B\mathbf{y}-\mathbf{b}\right) - \frac{\boldsymbol{\lambda}_k}{\beta}\right\|^2 \right\}$.
  \STATE[3.] $\boldsymbol{\lambda}_{k+1} \leftarrow \boldsymbol{\lambda}_{k} - \beta\left(A\mathbf{x}_{k+1} + B\mathbf{y}_{k+1} -\mathbf{b} \right)$.
\ENDFOR
\end{algorithmic}
\end{algorithm}

\section{Main Results of Convergence Rates}\label{sec:analysis}
In this section, we will show that our Stochastic ADMM given in Alg.\ref{alg:stochastic_admm} exhibits a rate $O(1/\sqrt{t})$ of convergence in terms of both the objective value \emph{and} the feasibility violation: $\mathbb{E}[\theta(\bar{\mathbf{u}}_t)-\theta(\mathbf{u}_*)+\rho\|A\bar{\mathbf{x}}_t+B\bar{\mathbf{y}}_t-\mathbf{b}\|_2]=O(1/\sqrt{t})$. We extend the main result if more structural information of $\theta_1$ is available.

Before we address the main theorem on convergence rates, we first present an upper bound of the variation of the Lagrangian function and its first order approximation based on each iteration points.
\begin{lemma}\label{lemma:VI_bound}
 $\forall\mathbf{w}\in\mathcal{W}, k\geq 1$, we have
\begin{equation}\label{eq:VI_bound}
\begin{split}
&\theta_1(\mathbf{x}_{k}) + \theta_2(\mathbf{y}_{k+1}) - \theta(\mathbf{u}) + (\mathbf{w}_{k+1} - \mathbf{w})^T F(\mathbf{w}_{k+1}) 
\leq \frac{\eta_{k+1}\| \theta_1^{\prime}(\mathbf{x}_{k},\boldsymbol{\xi}_{k+1}) \|^2}{2}\\
&+\frac{1}{2\eta_{k+1}}\left( \|\mathbf{x}_k-\mathbf{x}\|^2- \|\mathbf{x}_{k+1}-\mathbf{x}\|^2 \right) 
+ \frac{\beta}{2} \left(\|A\mathbf{x}+B\mathbf{y}_k-\mathbf{b}\|^2-\|A\mathbf{x}+B\mathbf{y}_{k+1}-\mathbf{b}\|^2\right) \\
&+\left\langle \boldsymbol{\delta}_{k+1}, \mathbf{x}-\mathbf{x}_{k} \right\rangle + \frac{1}{2\beta} \left( \|\boldsymbol{\lambda}-\boldsymbol{\lambda}_k\|_2^2 - 
\|\boldsymbol{\lambda}-\boldsymbol{\lambda}_{k+1}\|_2^2  \right).
\end{split}
\end{equation}
\end{lemma}

Utilizing this lemma we are able to obtain our main result shown as below. We present our main theorem of the convergence in two fashions, both in terms of expectation and probability satisfaction. 
\begin{theorem}\label{theorem:stocadmm_rate}
Let $\eta_k = \frac{D_{\mathcal{X}}}{M\sqrt{2k}},\forall k\geq 1$ and $\rho>0$.
\begin{enumerate}[(i)]
\item Under Assumption \ref{weak assumption}, we have $\forall t\geq 1$,
\begin{equation}\label{eq:expectation_bound}
\mathbb{E}[\theta(\bar{\mathbf{u}}_t)-\theta(\mathbf{u}_*)+\rho\|A\bar{\mathbf{x}}_t+B\bar{\mathbf{y}}_t-\mathbf{b}\|]\leq 
 M_1(t)+M_2(t)\equiv\frac{\sqrt{2}D_{\mathcal{X}}M}{\sqrt{t}}+\frac{\beta D_{\mathbf{y}_*,B}^2 +\rho^2/\beta}{2t},
\end{equation}
\item Under Assumption \ref{weak assumption} and \ref{strong assumption}, we have for any $\Omega>0$,
\begin{equation}\label{eq: probability_bound}
\Pr\left\{ \theta(\bar{\mathbf{u}}_t)-\theta(\mathbf{u}_*)+\rho\|A\bar{\mathbf{x}}_t+B\bar{\mathbf{y}}_t-\mathbf{b}\|>\left(1+\frac{1}{2}\Omega+2\sqrt{2\Omega}\right)M_1(t)+M_2(t)\right\}\leq 2\exp\{-\Omega\},
\end{equation}
\end{enumerate}
\end{theorem}

\begin{remark}
Adapting our proof techniques to the deterministic case where no noise takes place, we are able to obtain a similar result for deterministic ADMM:
\begin{equation}
\forall \rho >0, t\geq 1,\ \ \theta(\bar{\mathbf{u}}_t)-\theta(\mathbf{u}_*)+\rho\|A\bar{\mathbf{x}}_t+B\bar{\mathbf{y}}_t-\mathbf{b}\|_2\leq 
\frac{\beta D_{\mathbf{y}_*,B}^2}{2t} + \frac{\rho^2}{2\beta t},
\end{equation}
\end{remark}
While resulting in a $O(1/t)$ convergence rate same as the existing literature \cite{he12ocrdradm,he12oncrdradmm,wang12oadm}, the above finding is actually a  significant advance in the theoretical aspects of ADMM. For the first time, the convergence of ADMM is proved in terms of objective value and feasibility violation. By contrast, the existing literature \cite{he12ocrdradm,he12oncrdradmm,wang12oadm} only shows the convergence of ADMM in terms of the satisfaction of variational inequalities, which is not a direct measure of how fast an algorithm reaches the optimal solution.

\subsection{Extension: Strongly Convex $\theta_1$}
When function $\theta_1(\cdot)$ is strongly convex, the convergence rate of Stochastic ADMM can be improved to $O\left(\frac{\log t}{t}\right)$.
\begin{theorem}
When $\theta_1$ is $\mu$-strongly convex with respect to $\|\cdot\|$, taking $\eta_k = \frac{1}{k\mu }$ in Alg.\ref{alg:stochastic_admm}, under Assumption \ref{weak assumption}, $\forall \rho>0, t\geq 1$ we have $\mathbb{E}\left[\theta(\bar{\mathbf{u}}_t)-\theta(\mathbf{u}_*)+\rho\|A\bar{\mathbf{x}}_t+B\bar{\mathbf{y}}_t-\mathbf{b}\|_2\right] \leq \frac{M^2\log t}{\mu t} + \frac{\mu D_{\mathcal{X}}^2}{2t} + \frac{\beta D_{\mathbf{y}_*,B}^2}{2t}  + \frac{\rho^2}{2\beta t}$.
\end{theorem}
\subsection{Extension: Lipschitz Smooth $\theta_1$}
Since the bounds given in Theorem \ref{theorem:stocadmm_rate} are related to the magnitude of subgradients, they do not provide any intuition of the performance in low-noise scenarios. With a Lipschitz smooth function $\theta_1$, we are able to obtain convergence rates in terms of the variations of gradients, as stated in Assumption \ref{sigma_assumption}. Besides, under this assumption we are able to replace the unusual definition of $\bar{\mathbf{u}}_k$ in (\ref{eq:notations}) with
\begin{equation}\label{eq:def_u_new}
\bar{\mathbf{u}}_k \equiv \begin{pmatrix}
\frac{1}{k}\sum_{i=1}^{k}\mathbf{x}_i \\
\frac{1}{k}\sum_{i=1}^{k}\mathbf{y}_i \end{pmatrix}.
\end{equation}
\begin{theorem}
When $\theta_1(\cdot)$ is $L$-Lipschitz smooth with respect to $\|\cdot\|$, taking $\eta_k = \frac{1}{L+\sigma\sqrt{2k}/D_{\mathcal{X}}}$ in Alg.\ref{alg:stochastic_admm}, under Assumption \ref{sigma_assumption}, $\forall \rho > 0, t\geq 1$ we have $\mathbb{E}\left[\theta(\bar{\mathbf{u}}_t)-\theta(\mathbf{u}_*)+\rho\|A\bar{\mathbf{x}}_t+B\bar{\mathbf{y}}_t-\mathbf{b}\|_2\right] \leq  \frac{\sqrt{2}D_{\mathcal{X}}\sigma}{\sqrt{t}} +
\frac{LD^2_{\mathcal{X}}}{2t}+ \frac{\beta D_{\mathbf{y}_*,B}^2}{2t}  + \frac{\rho^2}{2\beta t}$.
\end{theorem}

\section{Summary and Future Work}
In this paper, we have proposed the stochastic setting for ADMM along with our stochastic ADMM algorithm. Based on a first-order approximation of the stochastic function, our algorithm is applicable to a very broad class of problems even with functions that have no closed-form solution to the subproblem of minimizing the augmented $\theta_1$. We have also established convergence rates under various structural assumptions of $\theta_1$: $O(1/\sqrt{t})$ for convex functions and $O(\log t/t)$ for strongly convex functions. We are working on integrating Nesterov's optimal first-order methods \cite{nesterov04ilco} to our algorithm, which will help in achieving optimal convergence rates. More interesting and challenging applications will be carried out in our future work.

\section{Appendix}
\subsection{3-Points Relation}
Before proving Lemma \ref{lemma:VI_bound}, we will start with the following simple lemma, which is a  very useful result by implementing Bregman divergence as a prox-function in proximal methods. 
\begin{lemma}\label{lemma:three_point}
Let $l(\mathbf{x}):\mathcal{X}\rightarrow \mathbb{R}$ be a convex differentiable function with gradient $\mathbf{g}$. Let scalar $s\geq 0$. For any vector $\mathbf{u}$ and $\mathbf{v}$, denote their Bregman divergence as $D(\mathbf{u},\mathbf{v}) \equiv \omega(\mathbf{u})-\omega(\mathbf{v}) - \langle \nabla\omega(\mathbf{v}),\mathbf{u}-\mathbf{v} \rangle$. If $\forall \mathbf{u}\in\mathcal{X}$,
\begin{equation}\label{eq:lm_prox_method}
\mathbf{x}^* \equiv \arg\min_{\mathbf{x}\in\mathcal{X}} l(\mathbf{x}) + sD(\mathbf{x},\mathbf{u}),
\end{equation}
then
\begin{equation*}
\left\langle \mathbf{g}(\mathbf{x}^*), \mathbf{x}^*-\mathbf{x} \right\rangle \leq s\left[D(\mathbf{x},\mathbf{u}) - D(\mathbf{x}, \mathbf{x}^*) - D(\mathbf{x}^*, \mathbf{u}) \right].
\end{equation*}
\end{lemma}
\begin{proof}
Invoking the optimality condition for (\ref{eq:lm_prox_method}), we have
\begin{equation*}
\left\langle \mathbf{g}(\mathbf{x}^*) + s\nabla D(\mathbf{x}^*, \mathbf{u}), \mathbf{x}-\mathbf{x}^* \right\rangle \geq 0,\  \forall \mathbf{x}\in\mathcal{X},
\end{equation*}
which is equivalent to
\begin{equation*}
\begin{split}
\left\langle \mathbf{g}(\mathbf{x}^*), \mathbf{x}^*-\mathbf{x} \right\rangle &\leq s \left\langle \nabla D(\mathbf{x}^*, \mathbf{u}), \mathbf{x}-\mathbf{x}^* \right\rangle\\
& = s \left\langle \nabla\omega(\mathbf{x}^*) -\nabla\omega(\mathbf{u}), \mathbf{x}-\mathbf{x}^*\right\rangle\\
& = s\left[D(\mathbf{x},\mathbf{u}) - D(\mathbf{x}, \mathbf{x}^*) - D(\mathbf{x}^*,\mathbf{u}) \right].
\end{split}
\end{equation*}
\end{proof}

\subsection {Proof of Lemma \ref{lemma:VI_bound}}
\begin{proof}
Due to the convexity of $\theta_1$ and using the definition of $\boldsymbol{\delta}_k$, we have
\begin{equation}\label{theta1_conv}
\theta_1(\mathbf{x}_{k}) - \theta_1(\mathbf{x})  \leq \left\langle \theta_1^{\prime}(\mathbf{x}_{k}), \mathbf{x}_{k} - \mathbf{x} \right\rangle 
= \left\langle \theta_1^{\prime}(\mathbf{x}_{k},\boldsymbol{\xi}_{k+1}), \mathbf{x}_{k+1} - \mathbf{x} \right\rangle +\left\langle \boldsymbol{\delta}_{k+1}, \mathbf{x}-\mathbf{x}_{k} \right\rangle
+\left\langle \theta_1^{\prime}(\mathbf{x}_{k},\boldsymbol{\xi}_{k+1}), \mathbf{x}_{k} - \mathbf{x}_{k+1} \right\rangle.
\end{equation}

Applying Lemma \ref{lemma:three_point} to Line 1 of Alg.\ref{alg:stochastic_admm} and taking $D(\mathbf{u},\mathbf{v})= \frac{1}{2}\|\mathbf{v}-\mathbf{u}\|^2$, we have
\begin{equation}\label{opt_cond_l1_derive}
\begin{split}
&\left\langle  \theta_1^{\prime}(\mathbf{x}_{k},\boldsymbol{\xi}_{k+1})  + A^T\left[\beta(A\mathbf{x}_{k+1}+B\mathbf{y}_k-\mathbf{b})-\boldsymbol{\lambda}_k \right],  \mathbf{x}_{k+1} -\mathbf{x}\right\rangle\\
&\leq \frac{1}{2\eta_{k+1}}\left(\|\mathbf{x}_k-\mathbf{x}\|^2- \|\mathbf{x}_{k+1}-\mathbf{x}\|^2 - \|\mathbf{x}_{k}-\mathbf{x}_{k+1}\|^2 \right)
\end{split}
\end{equation}

Combining (\ref{theta1_conv}) and (\ref{opt_cond_l1_derive}) we have
\begin{equation}\label{theta_1_int}
\begin{split}
&\ \ \ \ \theta_1(\mathbf{x}_{k}) - \theta_1(\mathbf{x}) + \left\langle \mathbf{x}_{k+1}-\mathbf{x}, -A^T \boldsymbol{\lambda}_{k+1}\right\rangle \\
&\overset{(\ref{theta1_conv})}{\leq} 
\left\langle \theta_1^{\prime}(\mathbf{x}_{k},\boldsymbol{\xi}_{k+1}), \mathbf{x}_{k+1} - \mathbf{x} \right\rangle +\left\langle \boldsymbol{\delta}_{k+1}, \mathbf{x}-\mathbf{x}_{k} \right\rangle
+\left\langle \theta_1^{\prime}(\mathbf{x}_{k},\boldsymbol{\xi}_{k+1}), \mathbf{x}_{k} - \mathbf{x}_{k+1} \right\rangle+\\
&\ \ \ \ \left\langle \mathbf{x}_{k+1}-\mathbf{x}, A^T \left[ \beta(A\mathbf{x}_{k+1}+B\mathbf{y}_{k+1}-\mathbf{b})-\boldsymbol{\lambda}_k \right]\right\rangle \\
&= \left\langle  \theta_1^{\prime}(\mathbf{x}_{k},\boldsymbol{\xi}_{k+1})  + A^T\left[\beta(A\mathbf{x}_{k+1}+B\mathbf{y}_k-\mathbf{b})-\boldsymbol{\lambda}_k \right],  \mathbf{x}_{k+1} -\mathbf{x}\right\rangle+\\
&\ \ \ \ \left\langle \boldsymbol{\delta}_{k+1}, \mathbf{x}-\mathbf{x}_{k} \right\rangle
+\left\langle \mathbf{x}-\mathbf{x}_{k+1}, \beta A^TB(\mathbf{y}_k-\mathbf{y}_{k+1})\right\rangle 
+\left\langle \theta_1^{\prime}(\mathbf{x}_{k},\boldsymbol{\xi}_{k+1}), \mathbf{x}_{k} - \mathbf{x}_{k+1} \right\rangle\\
& \overset{(\ref{opt_cond_l1_derive})}{\leq} \frac{1}{2\eta_{k+1}}\left(\|\mathbf{x}_k-\mathbf{x}\|^2- \|\mathbf{x}_{k+1}-\mathbf{x}\|^2 - \|\mathbf{x}_{k+1}-\mathbf{x}_k\|^2 \right) 
+ \left\langle \boldsymbol{\delta}_{k+1}, \mathbf{x}-\mathbf{x}_{k} \right\rangle + \\
&\ \ \ \ \left\langle \mathbf{x}-\mathbf{x}_{k+1}, \beta A^TB(\mathbf{y}_k-\mathbf{y}_{k+1})\right\rangle 
+\left\langle \theta_1^{\prime}(\mathbf{x}_{k},\boldsymbol{\xi}_{k+1}), \mathbf{x}_{k} - \mathbf{x}_{k+1} \right\rangle
\end{split}
\end{equation}
We handle the last two terms separately:
\begin{equation}\label{two_term_1}
\begin{split}
&\left\langle \mathbf{x}-\mathbf{x}_{k+1}, \beta A^TB(\mathbf{y}_k-\mathbf{y}_{k+1})\right\rangle = \beta \left\langle A\mathbf{x}-A\mathbf{x}_{k+1}, B\mathbf{y}_k-B\mathbf{y}_{k+1} \right\rangle\\
&= \frac{\beta}{2}\left[ \left(\|A\mathbf{x}+B\mathbf{y}_k-\mathbf{b}\|^2-\|A\mathbf{x}+B\mathbf{y}_{k+1}-\mathbf{b}\|^2\right) + \left(\|A\mathbf{x}_{k+1}+B\mathbf{y}_{k+1}-\mathbf{b}\|^2-\|A\mathbf{x}_{k+1}+B\mathbf{y}_{k}-\mathbf{b}\|^2\right)  \right]\\
&\leq \frac{\beta}{2} \left(\|A\mathbf{x}+B\mathbf{y}_k-\mathbf{b}\|^2- \|A\mathbf{x}+B\mathbf{y}_{k+1}-\mathbf{b}\|^2\right) + \frac{1}{2\beta}\|\boldsymbol{\lambda}_{k+1}-\boldsymbol{\lambda}_k\|^2
\end{split}
\end{equation}
and
\begin{equation}\label{two_term_2}
 \left\langle \theta_1^{\prime}(\mathbf{x}_{k},\boldsymbol{\xi}_{k+1}), \mathbf{x}_{k} - \mathbf{x}_{k+1} \right\rangle
\leq \frac{\eta_{k+1}\|\theta_1^{\prime}(\mathbf{x}_{k},\boldsymbol{\xi}_{k+1}) \|^2}{2} +  \frac{\|\mathbf{x}_k - \mathbf{x}_{k+1}\|^2}{2\eta_{k+1}},
\end{equation}
where the last step is due to Young's inequality. Inserting (\ref{two_term_1}) and (\ref{two_term_2}) into (\ref{theta_1_int}),  we have
\begin{equation}\label{eq:ineq_theta1}
\begin{split}
&\theta_1(\mathbf{x}_{k}) - \theta_1(\mathbf{x}) + \left\langle \mathbf{x}_{k+1}-\mathbf{x}, -A^T \boldsymbol{\lambda}_{k+1}\right\rangle \\
&\leq  \frac{1}{2\eta_{k+1}}\left( \|\mathbf{x}_k-\mathbf{x}\|^2- \|\mathbf{x}_{k+1}-\mathbf{x}\|^2 \right) +
 \frac{\eta_{k+1}\| \theta_1^{\prime}(\mathbf{x}_{k},\boldsymbol{\xi}_{k+1}) \|^2}{2}+\left\langle \boldsymbol{\delta}_{k+1}, \mathbf{x}-\mathbf{x}_{k} \right\rangle \\
& + \frac{\beta}{2} \left(\|A\mathbf{x}+B\mathbf{y}_k-\mathbf{b}\|^2-\|A\mathbf{x}+B\mathbf{y}_{k+1}-\mathbf{b}\|^2\right) + \frac{1}{2\beta}\|\boldsymbol{\lambda}_{k+1}-\boldsymbol{\lambda}_k\|^2 ,
\end{split}
\end{equation}

Due to the optimality condition of Line 2 in Alg.\ref{alg:stochastic_admm} and the convexity of $\theta_2$, we have
\begin{equation}\label{eq:ineq_theta2}
\theta_2(\mathbf{y}_{k+1}) - \theta_2(\mathbf{y}) + \left\langle \mathbf{y}_{k+1}-\mathbf{y}, -B^T \boldsymbol{\lambda}_{k+1}  \right\rangle \leq 0.
\end{equation}

Using Line 3 in Alg.\ref{alg:stochastic_admm}, we have
\begin{equation}\label{eq:ineq_lambda}
\begin{split}
&\left\langle \boldsymbol{\lambda}_{k+1} - \boldsymbol{\lambda}, A\mathbf{x}_{k+1} + B\mathbf{y}_{k+1}-\mathbf{b} \right\rangle\\
&= \frac{1}{\beta} \left\langle \boldsymbol{\lambda}_{k+1}-\boldsymbol{\lambda},  \boldsymbol{\lambda}_{k}-\boldsymbol{\lambda}_{k+1} \right\rangle\\
&=\frac{1}{2\beta} \left( \|\boldsymbol{\lambda}-\boldsymbol{\lambda}_k\|^2 - 
\|\boldsymbol{\lambda}-\boldsymbol{\lambda}_{k+1}\|^2 - 
\|\boldsymbol{\lambda}_{k+1}-\boldsymbol{\lambda}_{k}\|^2  \right)
\end{split}
\end{equation}

Taking the summation of inequalities (\ref{eq:ineq_theta1}) (\ref{eq:ineq_theta2}) and (\ref{eq:ineq_lambda}),  we obtain the result as desired.
\end{proof}

\subsection{Proof of Theorem 1}
\begin{proof}
$(i).$ Invoking convexity of $\theta_1(\cdot)$ and $\theta_2(\cdot)$ and the monotonicity of operator $F(\cdot)$, we have $\forall \mathbf{w}\in\Omega$:
\begin{equation}\label{eq:convex_theta}
\begin{split}
\theta(\bar{\mathbf{u}}_t) - \theta(\mathbf{u}) + (\bar{\mathbf{w}}_t-\mathbf{w})^T F(\bar{\mathbf{w}}_t)
 &\leq \frac{1}{t}\sum_{k=1}^t\left[ \theta_1(\mathbf{x}_{k-1}) + \theta_2(\mathbf{y}_{k})
- \theta(\mathbf{u}) + (\mathbf{w}_k-\mathbf{w})^T F(\mathbf{w}_k) \right]\\
&= \frac{1}{t}\sum_{k=0}^{t-1}\left[ \theta_1(\mathbf{x}_{k}) +\theta_2(\mathbf{y}_{k+1})
- \theta(\mathbf{u}) + (\mathbf{w}_{k+1}-\mathbf{w})^T F(\mathbf{w}_{k+1}) \right]\\
\end{split}
\end{equation}

Applying Lemma \ref{lemma:VI_bound} at the optimal solution $(\mathbf{x},\mathbf{y})=(\mathbf{x}_*,\mathbf{y}_*)$, we can derive from (\ref{eq:convex_theta}) that, $\forall \boldsymbol{\lambda}$
\begin{equation}\label{eq:main_pf_lemma_deri}
\begin{split}
&\ \ \ \ \theta(\bar{\mathbf{u}}_t) - \theta(\mathbf{u}_*)+(\bar{\mathbf{x}}_t-\mathbf{x}_*)^T(-A^T\bar{\boldsymbol{\lambda}}_t)
+(\bar{\mathbf{y}}_t-\mathbf{y}_*)^T(-B^T\bar{\boldsymbol{\lambda}}_t) + (\bar{\boldsymbol{\lambda}}_t-\boldsymbol{\lambda})^T(A\bar{\mathbf{x}}_t + B\bar{\mathbf{y}}_t-\mathbf{b} )\\
&\overset{(\ref{eq:VI_bound})}{\leq}\frac{1}{t}\sum_{k=0}^{t-1}\left[ \frac{\eta_{k+1}\| \theta_1^{\prime}(\mathbf{x}_{k},\boldsymbol{\xi}_{k+1}) \|^2}{2}+\frac{1}{2\eta_{k+1}}\left( \|\mathbf{x}_k-\mathbf{x}_*\|^2- \|\mathbf{x}_{k+1}-\mathbf{x}_*\|^2 \right) +\left\langle \boldsymbol{\delta}_{k+1}, \mathbf{x}_*-\mathbf{x}_{k} \right\rangle \right]\\
&\quad+\frac{1}{t}\left(\frac{\beta}{2}\|A\mathbf{x}_*+B\mathbf{y}_0-\mathbf{b}\|^2+ \frac{1}{2\beta} \|\boldsymbol{\lambda}-\boldsymbol{\lambda}_0\|^2\right)\\
&\leq \frac{1}{t}\sum_{k=0}^{t-1}\left[ \frac{\eta_{k+1} \| \theta_1^{\prime}(\mathbf{x}_{k},\boldsymbol{\xi}_{k+1}) \|^2}{2} +\left\langle \boldsymbol{\delta}_{k+1}, \mathbf{x}_*-\mathbf{x}_{k} \right\rangle \right]+\frac{1}{t}\left(\frac{D_{\mathcal{X}}^2}{2\eta_t}+\frac{\beta}{2}D_{\mathbf{y}_*,B}^2+ \frac{1}{2\beta} \|\boldsymbol{\lambda}-\boldsymbol{\lambda}_0\|_2^2\right)
\end{split}
\end{equation}

The above inequality is true for all $\boldsymbol{\lambda}\in \mathbb{R}^m$, hence it also holds in the ball $\mathcal{B}_0=\{\boldsymbol{\lambda}:\|\boldsymbol{\lambda}\|_2\leq \rho\}$. Combing with the fact that the optimal solution must also be feasible, it follows that
\begin{equation}\label{eq:main_pf_max_lambda}
\begin{split}
&\ \ \ \ \max_{\boldsymbol{\lambda}\in\mathcal{B}_0}\left\{\theta(\bar{\mathbf{u}}_t) - \theta(\mathbf{u}_*)+(\bar{\mathbf{x}}_t-\mathbf{x}_*)^T(-A^T\bar{\boldsymbol{\lambda}}_t)+(\bar{\mathbf{y}}_t-\mathbf{y}_*)^T(-B^T\bar{\boldsymbol{\lambda}}_t)
+(\bar{\boldsymbol{\lambda}}_t-\boldsymbol{\lambda})^T(A\bar{\mathbf{x}}_t+B\bar{\mathbf{y}}_t-\mathbf{b})\right\}\\
&=\max_{\boldsymbol{\lambda}\in\mathcal{B}_0}\left\{\theta(\bar{\mathbf{u}}_t) - \theta(\mathbf{u}_*)+
\bar{\boldsymbol{\lambda}}_t^T(A\mathbf{x}_*+B\mathbf{y}_*-b)-\boldsymbol{\lambda}^T(A\bar{\mathbf{x}}_t+B\bar{\mathbf{y}}_t-\mathbf{b})\right\}\\
&=\max_{\boldsymbol{\lambda}\in\mathcal{B}_0}\left\{\theta(\bar{\mathbf{u}}_t) - \theta(\mathbf{u}_*)-\boldsymbol{\lambda}^T(A\bar{\mathbf{x}}_t+B\bar{\mathbf{y}}_t-\mathbf{b})\right\}\\
&=\theta(\bar{\mathbf{u}}_t)-\theta(\mathbf{u}_*)+\rho\|A\bar{\mathbf{x}}_t+B\bar{\mathbf{y}}_t-\mathbf{b}\|_2
\end{split}
\end{equation}

Taking an expectation over (\ref{eq:main_pf_max_lambda}) and using (\ref{eq:main_pf_lemma_deri}) we have:
\begin{equation*}
\begin{split}
&\ \ \ \ \mathbb{E}\left[\theta(\bar{\mathbf{u}}_t)-\theta(\mathbf{u}_*)+\rho\|A\bar{\mathbf{x}}_t+B\bar{\mathbf{y}}_t-\mathbf{b}\|_2\right]\\
&\leq \mathbb{E}\left[\frac{1}{t}\sum_{k=0}^{t-1}\left( \frac{\eta_{k+1} \| \theta_1^{\prime}(\mathbf{x}_{k},\boldsymbol{\xi}_{k+1}) \|^2}{2} +\left\langle \boldsymbol{\delta}_{k+1}, \mathbf{x}_*-\mathbf{x}_{k} \right\rangle \right)+\frac{1}{t}\left(\frac{D_{\mathcal{X}}^2}{2\eta_t}+\frac{\beta}{2}D_{\mathbf{y}_*,B}^2\right) \right]\\
&\ \ \ \ +\mathbb{E}\left[\max_{\boldsymbol{\lambda}\in\mathcal{B}_0}  \left\{\frac{1}{2\beta t} \|\boldsymbol{\lambda}-\boldsymbol{\lambda}_0\|_2^2  \right\}\right]\\
&\leq\frac{1}{t}\left(\frac{M^2}{2}\sum_{k=1}^t\eta_k+\frac{D_{\mathcal{X}}^2}{2\eta_t}\right)+\frac{\beta D_{\mathbf{y}_*,B}^2}{2t} + \frac{\rho^2}{2\beta t}+\frac{1}{t}\sum_{k=0}^{t-1}\mathbb{E}\left[\left\langle \boldsymbol{\delta}_{k+1}, \mathbf{x}_*-\mathbf{x}_{k} \right\rangle\right]\\
&=\frac{1}{t}\left(\frac{M^2}{2}\sum_{k=1}^t\eta_k+\frac{D_{\mathcal{X}}^2}{2\eta_t}\right)+\frac{\beta D_{\mathbf{y}_*,B}^2}{2t} + \frac{\rho^2}{2\beta t}\\
&\leq\frac{\sqrt{2}D_{\mathcal{X}}M}{\sqrt{t}}+\frac{\beta D_{\mathbf{y}_*,B}^2}{2t} + \frac{\rho^2}{2\beta t}
\end{split}
\end{equation*}
In the second last step, we use the fact that $\mathbf{x}_{k}$ is independent of $\boldsymbol{\xi}_{k+1}$, hence $\mathbb{E}_{\boldsymbol{\xi}_{k+1}|\boldsymbol{\xi}_{[1:k]}}\left\langle \boldsymbol{\delta}_{k+1}, \mathbf{x}_*-\mathbf{x}_{k} \right\rangle = \left\langle  \mathbb{E}_{\boldsymbol{\xi}_{k+1}|\boldsymbol{\xi}_{[1:k]}} \boldsymbol{\delta}_{k+1},   \mathbf{x}_*-\mathbf{x}_{k}\right\rangle  = 0$.\\

$(ii)$ From the steps in the proof of part $(i)$, it follows that, 
\begin{equation}\label{eq: sum_split}
\begin{split}
&\ \ \ \ \theta(\bar{\mathbf{u}}_t)-\theta(\mathbf{u}_*)+\rho\|A\bar{\mathbf{x}}_t+B\bar{\mathbf{y}}_t-\mathbf{b}\|\\
&\leq \frac{1}{t}\sum_{k=0}^{t-1} \frac{\eta_{k+1}\  \| \theta_1^{\prime}(\mathbf{x}_{k},\boldsymbol{\xi}_{k+1}) \|^2}{2} +\frac{1}{t}\sum_{k=0}^{t-1}\left\langle \boldsymbol{\delta}_{k+1}, \mathbf{x}_*-\mathbf{x}_{k} \right\rangle +\frac{1}{t}\left(\frac{D_{\mathcal{X}}^2}{2\eta_t}+\frac{\beta}{2}D_{\mathbf{y}_*,B}^2+\frac{\rho^2}{2\beta}\right)\\
&\equiv A_t+B_t+C_t\\
\end{split}
\end{equation}
Note that random variables $A_t$ and $B_t$ are dependent on $\boldsymbol{\xi}_{[t]}$.

\begin{claim}
For $\Omega_1>0$, 
\begin{equation}\label{eq: A_bound}
\Pr\left(A_t\geq(1+\Omega_1)\frac{M^2}{2t}\sum_{k=1}^t\eta_k\right)\leq \exp\{-\Omega_1\}.
\end{equation}
\end{claim}

Let $\alpha_k\equiv \frac{\eta_k}{\sum_{k=1}^t\eta_k}\ \forall k=1,\ldots, t$, then $0\leq\alpha_k\leq 1$ and $ \sum_{k=1}^t\alpha_k=1$. Using the fact that $\{\boldsymbol{\delta}_{k},\forall k\}$ are independent and applying Assumption \ref{strong assumption}, one has
\begin{equation*}
\begin{split}
\E\left[\exp\left\{\sum_{k=1}^t \alpha_k\| \theta_1^{\prime}(\mathbf{x}_{k},\boldsymbol{\xi}_{k+1}) \|^2/M^2\right\}\right]
&=\prod_{k=1}^t\E\left[\exp\left\{\alpha_k \| \theta_1^{\prime}(\mathbf{x}_{k},\boldsymbol{\xi}_{k+1}) \|^2/M^2\right\}\right]\\
&\leq\prod_{k=1}^t \Big(\E\left[\exp\left\{ \| \theta_1^{\prime}(\mathbf{x}_{k},\boldsymbol{\xi}_{k+1}) \|^2/M^2\right\}\right]\Big)^{\alpha_k}  \qquad (\textrm{Jensen's Inequality})\\
&\leq\prod_{k=1}^t \left(\exp\{1\}\right)^{\alpha_k}
=\exp\left\{\sum_{k=1}^t\alpha_k\right\}
=\exp\{1\}
\end{split}
\end{equation*}

Hence, by Markov's Inequality, we can get
\begin{equation*}
\Pr\left(A_t\geq(1+\Omega_1)\frac{M^2}{2t}\sum_{k=1}^t\eta_k\right)
\leq \exp\left\{-(1+\Omega_1)\right\}\E\left[\exp\left\{\sum_{k=1}^t \alpha_k\| \theta_1^{\prime}(\mathbf{x}_{k},\boldsymbol{\xi}_{k+1}) \|^2/M^2\right\}\right]\leq\exp\{-\Omega_1\}.
\end{equation*}
\noindent We have therefore proved Claim 1. 


\begin{claim}
For $\Omega_2>0$, 
\begin{equation}\label{eq: B_bound}
\Pr\left(B_t>2\Omega_2\frac{D_{\mathcal{X}} M}{\sqrt{t}}\right)\leq \exp\left\{-\frac{\Omega_2^2}{4}\right\}.
\end{equation}
\end{claim}

In order to prove this claim, we adopt the following facts in Nemirovski's paper \cite{nemirovski09rsaasp}. 
\begin{lemma}
Given that for all $k=1,\ldots,t$, $\zeta_k$ is a deterministic function of $\boldsymbol{\xi}_{[k]}$ with $\E\left[\zeta_k|\boldsymbol{\xi}_{[k-1]}\right]=0$ and $\E\left[\exp\{\zeta_k^2/\sigma_k^2\}|\boldsymbol{\xi}_{[k-1]}\right]\leq\exp\{1\}$, we have
\begin{enumerate}[(a)]
\item For $\gamma\geq0$, $\E\left[\exp\{\gamma\zeta_k\}|\boldsymbol{\xi}_{[k-1]}\right]\leq\exp\{\gamma^2\sigma_k^2\},\forall k=1,\ldots, t $
\item Let $S_t=\sum_{k=1}^t\zeta_k$, then $\Pr\{S_t>\Omega\sqrt{\sum_{k=1}^t\sigma_k^2}\}\leq\exp\left\{-\frac{\Omega^2}{4}\right\}.$
\end{enumerate}
\end{lemma}
Using this result by setting $\zeta_k=\left\langle\boldsymbol{\delta}_{k},\mathbf{x}_*-\mathbf{x}_{k-1}\right\rangle$ , $S_t=\sum_{k=1}^t\zeta_k$, and $\sigma_k=2 D_{\mathcal{X}}M, \forall k$, we can verify that $\E\left[\zeta_k|\boldsymbol{\xi}_{[k-1]}\right]=0$ and 
\begin{equation*}
\E\left[\exp\{\zeta_k^2/\sigma_k^2\}|\boldsymbol{\xi}_{[k-1]}\right]\leq
\E\left[\exp\{   D_{\mathcal{X}}^2\|\boldsymbol{\delta}_k\|^2 /\sigma_k^2\}|\boldsymbol{\xi}_{[k-1]}\right]
\leq\exp\{1\},
\end{equation*}
since $|\zeta_k|^2\leq\|\mathbf{x}_*-\mathbf{x}_{k-1}\|^2 \|\boldsymbol{\delta}_k\|^2\leq D_{\mathcal{X}}^2\left(2\|\theta_1^{\prime}(\mathbf{x}_{k},\boldsymbol{\xi}_{k+1}) \|^2+2M^2\right)$.

Implementing the above results, it follows that
\begin{equation*}
\Pr\left(S_t>2\Omega_2 D_{\mathcal{X}}M\sqrt{t}\right)\leq \exp\left\{-\frac{\Omega_2^2}{4}\right\}.
\end{equation*}
Since $S_t=t B_t$, we have 
\begin{equation*}
\Pr\left(B_t>2\Omega_2\frac{D_{\mathcal{X}}M}{\sqrt{t}}\right)\leq \exp\left\{-\frac{\Omega_2^2}{4}\right\}
\end{equation*}
as desired.

Combining (\ref{eq: sum_split}), (\ref{eq: A_bound}) and (\ref{eq: B_bound}), we obtain
\begin{equation*}
\Pr\left(\textrm{Err}_{\rho}(\bar{\mathbf{u}}_t)>(1+\Omega_1)\frac{M^2}{2t}\sum_{k=1}^t\eta_k+2\Omega_2\frac{D_{\mathcal{X}}M}{\sqrt{t}}+C_t\right)
\leq\exp\left\{-\Omega_1\right\}+\exp\left\{-\frac{\Omega_2}{4}\right\},
\end{equation*}
where $\textrm{Err}_{\rho}(\bar{\mathbf{u}}_t)\equiv \theta(\bar{\mathbf{u}}_t)-\theta(\mathbf{u}_*)+\rho\|A\bar{\mathbf{x}}_t+B\bar{\mathbf{y}}_t-\mathbf{b}\|_2.$ Substituting $\Omega_1=\Omega, \Omega_2=2\sqrt{\Omega}$ and plugging in $\eta_k=\frac{D_{\mathcal{X}}}{M\sqrt{2k}}$, we obtain (\ref{eq: probability_bound}) as desired.
\end{proof}

\subsection{Proof of Theorem 2}
\begin{proof}
By the strong-convexity of $\theta_1$ we have $\forall \mathbf{x}$:
\begin{equation*}
\begin{split}
\theta_1(\mathbf{x}_{k}) - \theta_1(\mathbf{x})  &\leq \left\langle \theta_1^{\prime}(\mathbf{x}_{k}), \mathbf{x}_{k} - \mathbf{x} \right\rangle - \frac{\mu}{2}\|\mathbf{x}-\mathbf{x}_{k}\|^2\\
&= \left\langle \theta_1^{\prime}(\mathbf{x}_{k},\boldsymbol{\xi}_{k+1}), \mathbf{x}_{k+1} - \mathbf{x} \right\rangle +\left\langle \boldsymbol{\delta}_{k+1}, \mathbf{x}-\mathbf{x}_{k} \right\rangle 
+\left\langle  \theta_1^{\prime}(\mathbf{x}_{k},\boldsymbol{\xi}_{k+1}), \mathbf{x}_{k} - \mathbf{x}_{k+1}  \right\rangle
- \frac{\mu}{2}\|\mathbf{x}-\mathbf{x}_{k}\|^2.
\end{split}
\end{equation*}
Following the same derivations as in Lemma \ref{lemma:VI_bound} and Theorem \ref{theorem:stocadmm_rate} (i), we have
\begin{equation*}
\begin{split}
&\mathbb{E}\left[\theta(\bar{\mathbf{u}}_t)-\theta(\mathbf{u}_*)+\rho\|A\bar{\mathbf{x}}_t+B\bar{\mathbf{y}}_t-\mathbf{b}\|_2\right]\\
&\leq \mathbb{E}\left\{\frac{1}{t}\sum_{k=0}^{t-1}\left[ \frac{\eta_{k+1} \| \theta_1^{\prime}(\mathbf{x}_{k},\boldsymbol{\xi}_{k+1}) \|^2}{2} 
+\left(\frac{1}{2\eta_{k+1}} -\frac{\mu}{2}\right) \|\mathbf{x}_k-\mathbf{x}_*\|^2
-  \frac{ \|\mathbf{x}_{k+1}-\mathbf{x}_*\|^2}{2\eta_{k+1}} \right]\right\}\\
&+\frac{\beta D_{\mathbf{y}_*,B}^2}{2t}+\mathbb{E}\Big[\max_{\boldsymbol{\lambda}\in\mathcal{B}_0}  \Big\{\frac{1}{2\beta t} \|\boldsymbol{\lambda}-\boldsymbol{\lambda}_0\|_0^2  \Big\}\Big]\\
&\leq\frac{M^2}{2t}\sum_{k=1}^t \frac{1}{\mu k}+\frac{1}{t}\sum_{k=0}^{t-1}\mathbb{E}\left[\frac{\mu k}{2} \|\mathbf{x}_k-\mathbf{x}_*\|^2-  \frac{\mu(k+1)}{2}\|\mathbf{x}_{k+1}-\mathbf{x}_*\|^2 \right]
+\frac{\beta D_{\mathbf{y}_*,B}^2}{2t}  + \frac{\rho^2}{2\beta t}\\
&\leq\frac{M^2\log t}{\mu t} + \frac{\mu D_{\mathcal{X}}^2}{2t} + \frac{\beta D_{\mathbf{y}_*,B}^2}{2t}  + \frac{\rho^2}{2\beta t}.
\end{split}
\end{equation*}
\end{proof}

\subsection {Proof of Theorem 3}
\begin{proof}
The Lipschitz smoothness of $\theta_1$ implies that $\forall k\geq 0$:
\begin{equation*}
\begin{split}
\theta_1(\mathbf{x}_{k+1}) &\leq \theta_1(\mathbf{x}_k) + \left\langle \nabla\theta_1(\mathbf{x}_k), \mathbf{x}_{k+1}-\mathbf{x}_k \right\rangle + \frac{L}{2}\|\mathbf{x}_{k+1} - \mathbf{x}_k\|^2\\
&\overset{(\ref{eq:additional_notations})}{=} \theta_1(\mathbf{x}_k) + 
\left\langle \nabla\theta_1(\mathbf{x}_k,\boldsymbol{\xi}_{k+1}), \mathbf{x}_{k+1}-\mathbf{x}_k \right\rangle - \left\langle \boldsymbol{\delta}_{k+1}, \mathbf{x}_{k+1} - \mathbf{x}_k \right\rangle + \frac{L}{2}\|\mathbf{x}_{k+1} - \mathbf{x}_k\|^2.
\end{split}
\end{equation*}
It follows that $\forall \mathbf{x}\in\mathcal{X}$:
\begin{equation*}
\begin{split}
&\ \ \ \ \theta_1(\mathbf{x}_{k+1}) - \theta_1(\mathbf{x}) + \left\langle \mathbf{x}_{k+1}-\mathbf{x}, -A^T \boldsymbol{\lambda}_{k+1} \right\rangle\\
&\leq \theta_1(\mathbf{x}_k) - \theta_1(\mathbf{x})+ 
\left\langle \nabla\theta_1(\mathbf{x}_k,\boldsymbol{\xi}_{k+1}), \mathbf{x}_{k+1}-\mathbf{x}_k \right\rangle - \left\langle \boldsymbol{\delta}_{k+1}, \mathbf{x}_{k+1} - \mathbf{x}_k \right\rangle + \frac{L}{2}\|\mathbf{x}_{k+1} - \mathbf{x}_k\|^2 + \left\langle \mathbf{x}_{k+1}-\mathbf{x}, -A^T \boldsymbol{\lambda}_{k+1} \right\rangle\\
& = \theta_1(\mathbf{x}_k) - \theta_1(\mathbf{x})+ \left\langle \nabla\theta_1(\mathbf{x}_k,\boldsymbol{\xi}_{k+1}), \mathbf{x}-\mathbf{x}_k \right\rangle - \left\langle \boldsymbol{\delta}_{k+1}, \mathbf{x}_{k+1} - \mathbf{x}_k \right\rangle + \frac{L}{2}\|\mathbf{x}_{k+1} - \mathbf{x}_k\|^2 \\
&\ \ \ \ +\left[\left\langle \nabla\theta_1(\mathbf{x}_k,\boldsymbol{\xi}_{k+1}), \mathbf{x}_{k+1}-\mathbf{x} \right\rangle +\left\langle \mathbf{x}_{k+1}-\mathbf{x}, -A^T \boldsymbol{\lambda}_{k+1} \right\rangle\right]\\
&\leq \left\langle \nabla\theta_1(\mathbf{x}_k), \mathbf{x}_k - \mathbf{x}\right\rangle + \left\langle \nabla\theta_1(\mathbf{x}_k,\boldsymbol{\xi}_{k+1}), \mathbf{x}-\mathbf{x}_k \right\rangle - \left\langle \boldsymbol{\delta}_{k+1}, \mathbf{x}_{k+1} - \mathbf{x}_k \right\rangle + \frac{L}{2}\|\mathbf{x}_{k+1} - \mathbf{x}_k\|^2 \\
&\ \ \ \ +\left[\left\langle \nabla\theta_1(\mathbf{x}_k,\boldsymbol{\xi}_{k+1}), \mathbf{x}_{k+1}-\mathbf{x} \right\rangle +\left\langle \mathbf{x}_{k+1}-\mathbf{x}, -A^T \boldsymbol{\lambda}_{k+1} \right\rangle\right]\\
& = \left\langle \boldsymbol{\delta}_{k+1}, \mathbf{x} - \mathbf{x}_{k+1} \right\rangle + \frac{L}{2}\|\mathbf{x}_{k+1} - \mathbf{x}_k\|^2
+ \left[\left\langle \nabla\theta_1(\mathbf{x}_k,\boldsymbol{\xi}_{k+1}), \mathbf{x}_{k+1}-\mathbf{x} \right\rangle +\left\langle \mathbf{x}_{k+1}-\mathbf{x}, -A^T \boldsymbol{\lambda}_{k+1} \right\rangle\right]\\
& = \left\langle \boldsymbol{\delta}_{k+1}, \mathbf{x} - \mathbf{x}_{k+1} \right\rangle + \frac{L}{2}\|\mathbf{x}_{k+1} - \mathbf{x}_k\|^2 + \left\langle \mathbf{x}-\mathbf{x}_{k+1}, \beta A^TB(\mathbf{y}_k-\mathbf{y}_{k+1}) \right\rangle\\
&\ \ \ \ +\left\langle \nabla\theta_1(\mathbf{x}_k,\boldsymbol{\xi}_{k+1}) + A^T\left[\beta(A\mathbf{x_{k+1}}+B\mathbf{y}_k -\mathbf{b}) - \boldsymbol{\lambda}_k\right], \mathbf{x}_{k+1}-\mathbf{x} \right\rangle\\
&\overset{(\ref{opt_cond_l1_derive})}{\leq}
\frac{1}{2\eta_{k+1}}\left(\|\mathbf{x}-\mathbf{x}_{k}\|^2 -\|\mathbf{x}-\mathbf{x}_{k+1}\|^2  \right) - \frac{1/\eta_{k+1} - L}{2}\|\mathbf{x}_{k+1}-\mathbf{x}_{k}\|^2\\
&\ \ \ \ + \left\langle \mathbf{x}-\mathbf{x}_{k+1}, \beta A^TB(\mathbf{y}_k-\mathbf{y}_{k+1}) \right\rangle + \left\langle \boldsymbol{\delta}_{k+1}, \mathbf{x} - \mathbf{x}_{k+1} \right\rangle.
\end{split}
\end{equation*}
The last inner product can be bounded as below using Young's inequality, given that $\eta_{k+1} \leq \frac{1}{L}$:
\begin{equation*}
\begin{split}
\left\langle \boldsymbol{\delta}_{k+1}, \mathbf{x} - \mathbf{x}_{k+1} \right\rangle &= 
\left\langle \boldsymbol{\delta}_{k+1},  \mathbf{x} - \mathbf{x}_{k}\right\rangle + 
\left\langle \boldsymbol{\delta}_{k+1},  \mathbf{x}_{k} - \mathbf{x}_{k+1}\right\rangle\\
& \leq \left\langle \boldsymbol{\delta}_{k+1},  \mathbf{x} - \mathbf{x}_{k}\right\rangle + 
\frac{1}{2\left(1/\eta_{k+1}-L\right)}\|\boldsymbol{\delta}_{k+1}\|^2  + \frac{1/\eta_{k+1}-L}{2}\|\mathbf{x}_k - \mathbf{x}_{k+1}\|^2.
\end{split}
\end{equation*}
Combining this with inequalities (\ref{two_term_1},\ref{eq:ineq_theta2}) and (\ref{eq:ineq_lambda}), we can get a similar statement as that of Lemma \ref{lemma:VI_bound}:
\begin{equation*}
\begin{split}
&\theta(\mathbf{u}_{k+1}) - \theta(\mathbf{u}) + (\mathbf{w}_{k+1} - \mathbf{w})^T F(\mathbf{w}_{k+1}) 
\leq \frac{\| \boldsymbol{\delta}_{k+1} \|^2}{2(1/\eta_{k+1}-L)}\\
&+\frac{1}{2\eta_{k+1}}\left( \|\mathbf{x}_k-\mathbf{x}\|^2- \|\mathbf{x}_{k+1}-\mathbf{x}\|^2 \right) 
+ \frac{\beta}{2} \left(\|A\mathbf{x}+B\mathbf{y}_k-\mathbf{b}\|^2-\|A\mathbf{x}+B\mathbf{y}_{k+1}-\mathbf{b}\|^2\right) \\
&+\left\langle \boldsymbol{\delta}_{k+1}, \mathbf{x}-\mathbf{x}_{k} \right\rangle + \frac{1}{2\beta} \left( \|\boldsymbol{\lambda}-\boldsymbol{\lambda}_k\|_2^2 - 
\|\boldsymbol{\lambda}-\boldsymbol{\lambda}_{k+1}\|_2^2  \right).
\end{split}
\end{equation*}
The rest of the proof are essentially the same as Theorem \ref{theorem:stocadmm_rate} (i), except that we use the new definition of $\bar{\mathbf{u}}_k$ in (\ref{eq:def_u_new}).
\end{proof}

\bibliographystyle{IEEEbib}
\bibliography{Arxiv_OPT_StocADMM}
\end{document}